\documentclass[runningheads]{llncs}

\newcommand {\ent} {\mathrel{{\scriptstyle\mid\!\sim}}}

\newcommand {\sx} {\langle}
\newcommand {\dx} {\rangle}

\newcommand {\emme} {\mathcal{M}}
\newcommand {\enne} {\mathcal{N}}

\newcommand {\tc} {\mid}
\newcommand {\vuoto} {\emptyset}

\newcommand{\tip}{{\bf T}}

\newcommand{\lc}{\mathcal{LC}}
\newcommand{\alc}{\mathcal{ALC}}

\newcommand{\el}{\mathcal{EL}}
\newcommand{\elpb}{{\mathcal{EL}}^{+}_{\bot}}

\newcommand{\be}{\begin{enumerate}}
\newcommand{\ee}{\end{enumerate}}

\newcommand{\hide}[1]{}

\def \cases{\left \{\begin{array}{l}}
\def \endcases{\end{array}\right .}

\newcommand {\ri} {\rightarrow}
\newcommand {\Ri} {\Rightarrow}

\newcommand {\bes} {\begin{description}}
\newcommand{\ens} {\end{description}}
\newcommand {\la} {\langle}
\newcommand {\ra} {\rangle}

\newcommand {\beq} {\begin{quote}}
\newcommand {\enq} {\end{quote}}
\newcommand {\bit} {\begin{itemize}}
\newcommand {\enit} {\end{itemize}}

\newenvironment{pozz}{\color{black}}{\color{black}}
\usepackage{times}
\usepackage{helvet}
\usepackage{courier}
\usepackage{times}
\usepackage{undertilde}
\usepackage{graphicx}
\usepackage{latexsym}
\usepackage{graphicx}
\usepackage{color}
\usepackage{amssymb}
\usepackage{colortbl}
\usepackage{amsmath}
\usepackage{microtype}
\usepackage{amsmath}
\usepackage{tikz}

\def \ri{\rightarrow}
\def \Ri{\Rightarrow}

\begin{document}
\bibliographystyle{plain}

\title{Weighted defeasible knowledge bases and \\
a multipreference semantics \\ 
for a deep neural network model}



\author{Laura Giordano \and Daniele Theseider Dupr{\'{e}} }

\institute{DISIT - Universit\`a del Piemonte Orientale, 
 Alessandria, Italy  
}

\authorrunning{ }
\titlerunning{ }

 \maketitle
 

\begin{abstract}
In this paper we investigate the relationships between a multipreferential semantics for defeasible reasoning in knowledge representation and  a deep neural network model. Weighted knowledge bases for description logics are considered under a ``concept-wise" multipreference semantics. The semantics is further extended to fuzzy interpretations and exploited to provide a preferential interpretation of Multilayer Perceptrons. 

\end{abstract}

\vspace{-0.1cm}
\section{Introduction}
\vspace{-0.1cm}

Preferential approaches have been used to provide axiomatic foundations of non-mono- tonic  and 
common sense reasoning 
\cite{Delgrande:87,Makinson88,Pearl:88,KrausLehmannMagidor:90,Pearl90,whatdoes,BenferhatIJCAI93}.
They have been extended to description logics (DLs), to deal with inheritance with exceptions in ontologies,
by allowing for non-strict forms of inclusions,
called {\em typicality or defeasible inclusions}, 
with different preferential semantics \cite{lpar2007,sudafricaniKR,FI09},
and closure constructions \cite{casinistraccia2010,CasiniDL2013,AIJ15,Pensel18,CasiniStracciaM19,AIJ2020}. 


In this paper, we exploit a concept-wise multipreference semantics  
as a semantics for weighted knowledge bases,  i.e. knowledge bases in which defeasible or typicality inclusions of the form $\tip(C) \sqsubseteq D$ (meaning ``the typical $C$'s are $D$'s" or ``normally $C$'s are $D$'s") are given a positive or negative weight. 
This multipreference semantics, which takes into account preferences with respect to different concepts,  has been first introduced  as a semantics for ranked DL knowledge bases  \cite{iclp2020}. 
For weighted knowledge bases, we develop a different semantic closure construction, although in the spirit of other semantic constructions in the literature. We further extend the multipreference semantics to the fuzzy case.

The concept-wise multipreference semantics 
has been shown to have some desired properties from the knowledge representation point of view \cite{iclp2020,NMR2020}, and a related semantics with multiple preferences has also been proposed in the first-order logic setting by Delgrande and Rantsaudis \cite{Delgrande2020}. In previous work  \cite{CILC2020}, the concept-wise multipreference semantics has been used to provide a preferential interpretation of Self-Organising Maps \cite{kohonen2001}, psychologically and biologically plausible neural network models.
In this paper, we aim at investigating its relationships with another neural network model, Multilayer Perceptrons.

We consider a multilayer neural network after the training phase, when the synaptic weights have been learned, 
to show that  the neural network can be given a preferential DL semantics with multiple preferences, as well as a semantics based on fuzzy DL \linebreak interpretations 
and another one combining fuzzy interpretations with multiple preferences. 
The three semantics allow the input-output behavior of the network to be captured by interpretations built over a set of input stimuli through a simple construction, which exploits the activity level of neurons for the stimuli. 
Logical properties can be verified over such models by model checking.


To prove that the fuzzy multipreference interpretations, built from the network for a given set of input stimuli, 
are models of the neural network in a logical sense, we map the multilayer network to a conditional knowledge base, i.e., a set of weighted defeasible inclusions. 
We also provide an account of our approach in probabilistic DLs.

A logical interpretation of a neural network  
can be useful from the point of view of explainability, in view of a trustworthy, reliable and explainable AI \cite{Adadi18,Guidotti2019,Arrieta2020}, 
and can potentially be exploited as the basis for an integrated use of symbolic reasoning 
and neural models.

\section{The description logic $\alc$ and $\el$}  \label{sec:ALC}

In this section we recall the syntax and semantics of the description logic $\alc$ \cite{handbook} and of its lightweight fragment $\el$ \cite{rifel}  at the basis of OWL2 EL Profile.

Let ${N_C}$ be a set of concept names, ${N_R}$ a set of role names
  and ${N_I}$ a set of individual names.  
The set  of $\alc$ \emph{concepts} (or, simply, concepts) can be
defined inductively: \\ 
- $A \in N_C$, $\top$ and $\bot$ are {concepts};\\
- if $C$ and $ D$ are concepts, and $r \in N_R$, then $C \sqcap D,\; C \sqcup D,\; \neg C, \; \forall r.C,\; \exists r.C$ 
are {concepts}.

%

\noindent
A knowledge base (KB) $K$ is a pair $({\cal T}, {\cal A})$, where ${\cal T}$ is a TBox and
${\cal A}$ is an ABox.
The TBox ${\cal T}$ is  a set of concept inclusions (or subsumptions) $C \sqsubseteq D$, where $C,D$ are concepts.
The  ABox ${\cal A}$ is  a set of assertions of the form $C(a)$ and $r(a,b)$
where $C$ is a  concept, $a$ an individual name in $N_I$ and $r$ a role name in $N_R$.

An  $\alc$ {\em interpretation}  is defined as a pair $I=\langle \Delta, \cdot^I \rangle$ where:
$\Delta$ is a domain---a set whose elements are denoted by $x, y, z, \dots$---and 
$\cdot^I$ is an extension function that maps each
concept name $C\in N_C$ to a set $C^I \subseteq  \Delta$, 
each role name $r \in N_R$ to  a binary relation $r^I \subseteq  \Delta \times  \Delta$,
and each individual name $a\in N_I$ to an element $a^I \in  \Delta$.
It is extended to complex concepts  as follows:

$\top^I=\Delta$,  \ \ \ \ \ \ \  $\bot^I=\vuoto$,  \ \ \ \ \ \ \  $(\neg C)^I=\Delta \backslash C^I$,

$(\exists r.C)^I =\{x \in \Delta \tc \exists y.(x,y) \in r^I \ \mbox{and}  \ y \in C^I\}$,   \ \ \ \ \ \ \ $(C \sqcap D)^I=C^I \cap D^I$, 

$(\forall r.C)^I =\{x \in \Delta \tc \forall y. (x,y) \in r^I \Ri y \in C^I\}$,   \ \ \ \ \ \ \ \ \ $(C \sqcup D)^I=C^I \cup D^I$.


\noindent
The notion of satisfiability of a KB  in an interpretation and the notion of entailment are defined as follows:

\begin{definition}[Satisfiability and entailment] \label{satisfiability}
Given an $\lc$ interpretation $I=\langle \Delta, \cdot^I \rangle$: 

	- $I$  satisfies an inclusion $C \sqsubseteq D$ if   $C^I \subseteq D^I$;
	
	-   $I$ satisfies an assertion $C(a)$ (resp., $r(a,b)$) if $a^I \in C^I$ (resp.,  $(a^I,b^I) \in r^I$).
	

\noindent
 Given  a KB $K=({\cal T}, {\cal A})$, 
 an interpretation $I$  satisfies ${\cal T}$ (resp. ${\cal A}$) if $I$ satisfies all  inclusions in ${\cal T}$ (resp. all assertions in ${\cal A}$);
 $I$ is a \emph{model} of $K$ if $I$ satisfies ${\cal T}$ and ${\cal A}$.

 A subsumption $F= C \sqsubseteq D$ (resp., an assertion $C(a)$, $r(a,b)$),   {is entailed by $K$}, written $K \models F$, if for all models $I=$$\sx \Delta,  \cdot^I\dx$ of $K$,
$I$ satisfies $F$.

\end{definition}
Given a knowledge base $K$,
the {\em subsumption} problem is the problem of deciding whether an inclusion $C \sqsubseteq D$ is entailed by  $K$.

In the logic $\el$ \cite{rifel},
concepts are restricted to  $C \ \ := A \tc \top \tc C \sqcap C \tc \exists r.C $, i.e., union, complement and universal restriction are not $\el$ constructs. In the following, we will also consider the boolean fragment of $\alc$ only including constructs $\sqcap$, $\sqcup$, $\neg$.

\section{Fuzzy description logics} \label{sec:fuzzyDL}

Fuzzy description logics have been widely studied in the literature for representing vagueness in DLs \cite{Straccia05,Stoilos05,LukasiewiczStraccia09,PenalosaARTINT15,BobilloStracciaEL18}, 
based on the idea that concepts and roles can be interpreted 
as fuzzy sets and fuzzy binary relations.

As in Mathematical Fuzzy Logic \cite{Cintula2011} a formula has a degree of truth in an interpretation, rather than being either true or false,
in a fuzzy DL axioms 
are associated with a degree of truth (usually in the interval  $[0, 1]$). 
In the following we shortly recall the semantics of a fuzzy extension of $\alc$ referring to the survey by Lukasiewicz and Straccia \cite{LukasiewiczStraccia09}.
%
We limit our consideration 
to a few features of a fuzzy DL 
and, in particular, we omit considering datatypes.

A {\em fuzzy interpretation} for $\alc$ is a pair $I=\langle \Delta, \cdot^I \rangle$ where:
$\Delta$ is a non-empty domain and 
$\cdot^I$ is {\em fuzzy interpretation function} that assigns to each
concept name $A\in N_C$ a function  $A^I :  \Delta \ri [0,1]$,
to each role name $r \in N_R$  a function  $r^I:   \Delta \times  \Delta \ri [0,1]$,
and to each individual name $a\in N_I$ an element $a^I \in  \Delta$.
A domain element $x \in \Delta$ 
belongs to the extension of $A$ to some degree in $[0, 1]$, i.e., $A^I$ is a fuzzy set.

The  interpretation function $\cdot^I$ is extended to complex concepts as follows: 

$\mbox{\ \ \ }$ $\top^I(x)=1$, $\mbox{\ \ \ \ \ \ \ \ \ \  \ \ \ \ \ }$ $\bot^I(x)=0$,  $\mbox{\ \ \ \ \ \ \ \ \ \ \ \ \ \ \  \ \ \ \ \ }$  $(\neg C)^I(x)= \ominus C^I(x)$, 


$\mbox{\ \ \ }$  $(\exists r.C)^I(x) = sup_{y \in \Delta} \; \; r^I(x,y) \otimes C^I(y)$,  $\mbox{\ \ \ \ \ \ }$  $(C \sqcup D)^I(x) =C^I(x) \oplus D^I(x)$ 

$\mbox{\ \ \ }$  $(\forall r.C)^I (x) = inf_{y \in \Delta} \; \; r^I(x,y) \rhd C^I(y)$, $\mbox{\ \ \ \ \ \ \ }$  $(C \sqcap D)^I(x) =C^I(x) \otimes D^I(x)$ 

\noindent
where  $x \in \Delta$ and $\otimes$, $\oplus$, $\rhd$ and $\ominus$ are arbitrary but fixed t-norm, s-norm, implication function, and negation function, chosen among the combination functions of various fuzzy logics 
(we refer to \cite{LukasiewiczStraccia09} for details).

The  interpretation function $\cdot^I$ is also extended  to non-fuzzy axioms (i.e., to strict inclusions and assertions of an $\alc$ knowledge base) as follows:\\
 $(C \sqsubseteq D)^I= inf_{x \in \Delta}  C^I(x) \rhd D^I(x)$,
$\mbox{\ \ \ }$  $(C(a))^I=C^I(a^I)$,  $\mbox{\ \ \ }$  $(R(a,b))^I=R^I(a^I,b^I)$.

A {\em fuzzy $\alc$ knowledge base} $K$ is a pair $({\cal T}, {\cal A})$ where ${\cal T}$ is a fuzzy TBox  and ${\cal A}$ a fuzzy ABox. A fuzzy TBox is a set of {\em fuzzy concept inclusions} of the form $C \sqsubseteq D \;\theta\; n$, where $C \sqsubseteq D$ is an $\alc$ concept inclusion axiom, $\theta \in \{\geq,\leq,>,<\}$ and $n \in [0,1]$. A fuzzy ABox ${\cal A}$ is a set of {\em fuzzy assertions} of the form $C(a) \theta n$ or $r(a,b) \theta n$, where $C$ is an $\alc$ concept, $r\in N_R$, $a,b \in N_I$,  $\theta \in \{{\geq,}\leq,>,<\}$ and $n \in [0,1]$.
Following Bobillo and Straccia  \cite{BobilloStraccia2018}, we assume that fuzzy interpretations are {\em witnessed}, i.e., the sup and inf are attained at some point of the involved domain.
The notions of satisfiability of a KB  in a fuzzy interpretation and of entailment are defined in the natural way.
\begin{definition}[Satisfiability and entailment for fuzzy KBs] \label{satisfiability}
A  fuzzy interpretation $I$ satisfies a fuzzy $\alc$ axiom $E$ (denoted $I \models E$), as follows, 
 for $\theta \in \{\geq,\leq,>,<\}$:

- $I$ satisfies a fuzzy $\alc$ inclusion axiom $C \sqsubseteq D \;\theta\; n$ if $(C \sqsubseteq D)^I \theta\; n$;

- $I$ satisfies a fuzzy $\alc$ assertion $C(a) \; \theta \; n$ if $C^I(a^I) \theta\; n$;
 
- $I$ satisfies a fuzzy $\alc$ assertion $r(a,b) \; \theta \; n$ if $r^I(a^I,b^I) \theta\; n$.

\noindent
Given  a fuzzy KB $K=({\cal T}, {\cal A})$,
 a fuzzy interpretation $I$  satisfies ${\cal T}$ (resp. ${\cal A}$) if $I$ satisfies all fuzzy  inclusions in ${\cal T}$ (resp. all fuzzy assertions in ${\cal A}$).
A fuzzy interpretation $I$ is a \emph{model} of $K$ if $I$ satisfies ${\cal T}$ and ${\cal A}$.
A fuzzy axiom $E$   {is entailed by a fuzzy knowledge base $K$}, written $K \models E$, if for all models $I=$$\sx \Delta,  \cdot^I\dx$ of $K$,
$I$ satisfies $E$.
\end{definition}

%


\section{A concept-wise multipreference semantics for weighted KBs 
}

In this section  we develop an extension of $\el$ with defeasible inclusions having positive and negative weights, based on a concept-wise multipreference semantics first introduced for ranked $\elpb$ knowledge bases \cite{iclp2020}, where defeasible inclusions have positive integer ranks.
In addition to standard $\el$ inclusions $C \sqsubseteq D$ (called  {\em strict} inclusions in the following), the TBox ${\cal T}$ will also contain typicality inclusions of the form $\tip(C) \sqsubseteq D$, where $C$ and $D$ are $\el$ concepts.
A typicality inclusion $\tip(C) \sqsubseteq D$ means that ``typical C's are D's" or ``normally C's are D's" and corresponds to a conditional implication $C \ent D$ in Kraus, Lehmann and Magidor's (KLM) preferential approach \cite{KrausLehmannMagidor:90,whatdoes}. 
Such inclusions are defeasible, i.e.,  admit exceptions, while 
strict inclusions must be satisfied by all domain elements.
We assume that with each typicality inclusion is associated a weight $w$, a real number.
A positive weight supports the plausibility of a defeasible inclusion; a negative weight supports its  implausibility.

\subsection{Weighted $\el$ knowledge bases} \label{sec:WeightedBKs}


Let ${\cal C}= \{C_1, \ldots, C_k\}$ be a set of distinguished $\el$ concepts, the concepts for which defeasible inclusions are defined. 
A weighted TBox ${\cal T}_{C_i}$ is defined for each distinguished concept $C_i \in {\cal C}$ as a set of defeasible inclusions of the form $\tip(C_i) \sqsubseteq D$ with a weight.

A {\em weighted $\el$ knowledge base $K$ over ${\cal C}$} is a tuple $\langle  {\cal T}_{strict}, {\cal T}_{C_1}, \ldots, {\cal T}_{C_k}, {\cal A}  \rangle$, 
where ${\cal T}_{strict}$ is a set of strict concept inclusions, ${\cal A}$ is an ABox and, for each $C_j \in {\cal C}$,  ${\cal T}_{C_j}$ is a weighted TBox of defeasible inclusions,
$\{(d^i_h,w^i_h)\}$, where  each  $d^i_h$ is a typicality inclusion of the form $\tip(C_i) \sqsubseteq D_{i,h}$,  having weight $w^i_h$, a real number.

Consider, for instance, the ranked knowledge base $K =\langle {\cal T}_{strict},  {\cal T}_{Employee}, {\cal T}_{Student},$ $ {\cal A} \rangle$, over the set of distinguished concepts ${\cal C}=\{\mathit{Employee, Student}\}$, with empty ABox,
and with $ {\cal T}_{strict}$ containing the set of strict inclusions:

$\mathit{Employee  \sqsubseteq  Adult}$ \ \ \ \ \ \ \ \ \ \ \ \ \ \ \   $\mathit{Adult  \sqsubseteq  \exists has\_SSN. \top}$  
 \ \ \ \ \ \ \ \ \ \ \   $\mathit{PhdStudent  \sqsubseteq  Student}$ 


\noindent
The weighted TBox ${\cal T}_{Employee} $ 
contains the following weighted defeasible inclusions:

$(d_1)$ $\mathit{\tip(Employee) \sqsubseteq Young}$, \ \ - 50  \ \ \ \ \ \ \ \ \  \ \ \ \ \  

$(d_2)$ $\mathit{\tip(Employee) \sqsubseteq \exists has\_boss.Employee}$, \ \ 100


$(d_3)$ $\mathit{\tip(Employee) \sqsubseteq  \exists has\_classes.\top}$, \ \ -70;

\noindent
the weighted TBox ${\cal T}_{Student}$ contains the defeasible inclusions:

$(d_4)$ $\mathit{\tip(Student) \sqsubseteq Young}$, \ \ 90

$(d_5)$ $\mathit{\tip(Student) \sqsubseteq  \exists has\_classes.\top}$, \ \ 80   \ \ \ \ \ \ \ \ \ \  

$(d_6)$ $\mathit{\tip(Student) \sqsubseteq  \exists hasScholarship.\top}$, \ \  -30



%
%
%
%
%
%
%
\noindent
The meaning is that, while an employee normally has a boss, he is not likely to be young or have classes. Furthermore, between the two defeasible inclusions $(d_1)$ and $(d_3)$, the second one is considered less plausible than the first one. 
Given two employees Tom and Bob such that  Tom is not young, has no boss and has classes, while  Bob is not young, has a boss and has no classes, in the following, considering the weights above, we will regard Bob as being more typical than Tom as an employee. 



\subsection{The concept-wise preferences from weighted knowledge bases} \label{sec:multipref}

The concept-wise multipreference  semantics has been recently introduced as a semantics for ranked $\elpb$ knowledge bases   \cite{iclp2020}, 
 which 
 are inspired by Brewka's framework of basic preference descriptions  \cite{Brewka04}. 
For each concept $C_i \in {\cal C}$, a preference relation $<_{C_i}$ describes the preference among domain elements with respect to concept $C_i$. 
Each $<_{C_i}$ has the properties of preference relations in KLM-style ranked interpretations \cite{whatdoes}, that is,  $<_{C_i}$ is a modular and well-founded strict partial order. 
In particular, $<_{C_i}$ is {\em well-founded} 
if, for all $S \subseteq \Delta$, if $S\neq \emptyset$, then $min_{<_{C_i}}(S)\neq \emptyset$;
    $<_{C_i}$ is {\em modular} if,
for all $x,y,z \in \Delta$, $x <_{C_j} y$ implies ($x <_{C_j} z$ or $z <_{C_j} y$).

In the following we will recall the concept-wise semantics for $\alc$, which extends to its fragments considered in the following.
An $\alc$ interpretation, is extended with a collection of preference relations, one for each concept in ${\cal C}$. 
  \begin{definition}[Multipreference interpretation]\label{defi:multipreference}  
A {\em multipreference interpretation}  is a tuple
$\emme= \langle \Delta, <_{C_1}, \ldots, <_{C_k}, \cdot^I \rangle$, where:

(a) $\Delta$ is a 
domain, and $\cdot^I$ an interpretation function, as in $\alc$ interpretations;
 
(b) the $<_{C_i}$ are irreflexive, transitive, well-founded and modular relations over $\Delta$.


\end{definition}
The preference relation $<_{C_i}$ determines the relative typicality of domain individuals with respect to concept $C_i$.
For instance,  Tom may be more typical than Bob as a student ($\mathit{tom <_\mathit{Student} bob}$), but more exceptional as an employee ( $\mathit{bob <_\mathit{Employee} tom}$). 
The minimal $C_i$-elements with respect to $<_{C_i}$ 
are regarded as the most typical  $C_i$-elements.

While preferences do not need to agree, arbitrary conditional formulas cannot be evaluated with respect to a single preference relation. For instance, evaluating the inclusion ``Are typical employed students  young?" would require both the preferences $ <_\mathit{Student} $ and $ <_\mathit{Employee} $ to be considered. 
The approach proposed in \cite{iclp2020} 
is that of {\em combining} the preference relations $<_{C_i}$ into a single {\em global preference} relation $<$,
and than exploit the global preference for interpreting the typicality operator $\tip$, which may be applied to arbitrary concepts. 
A natural way to define the notion of global preference $<$ is by Pareto combination of the relations $<_{C_1}, \ldots,<_{C_k}$,
as follows:\\
$\mbox{ \ \ \ \ \ \  \ \ \ \ \ \ }$ $x <y  \mbox{ iff \ \ } $
$(i) \  x <_{C_i} y, \mbox{ for some } C_i \in {\cal C}, \mbox{ and }$ \\
$\mbox{ \ \ \ \ \ \  \ \ \ \ \ \ \ \ \ \ \ \ \ \ \ \ \ \ \ \ \ \ \ \ \ }  (ii)  \ \mbox{  for all } C_j\in {\cal C}, \;  x \leq_{C_j} y $.\\
A slightly more sophisticated notion of preference combination, which exploits a modified  Pareto condition
taking into account the specificity relation among concepts (such as, for instance, the fact that concept $\mathit{PhdStudent}$ is more specific than concept $\mathit{Student}$), has been considered for ranked knowledge bases \cite{iclp2020}. 

The addition of the global preference relation, leads to the definition of a notion of {\em concept-wise multipreference interpretation}, where concept $\tip(C)$ is interpreted as the set of all $<$-minimal  $C$ elements.
\begin{definition}\label{defi:cw-multipref}
A {\em concept-wise multipreference interpretation} (or cw$^m$-interpretation) 
 is a multipreference interpretation  $\emme= \langle \Delta, <_{C_1}, \ldots,<_{C_k}, <, \cdot^I \rangle$, according to Definition \ref{defi:multipreference}, such that
 the global preference relation $<$ is defined as above and $(\tip(C))^I = min_{<}(C^I)$,
where $Min_<(S)= \{u: u \in S$ and $\nexists z \in S$ s.t. $z < u \}$.
\end{definition}

\noindent
In the following, we define a notion of cw$^m$-model of a weighted $\el$ knowledge base $K$ as a cw$^m$-interpretation in which the preference relations $<_{C_i}$ are constructed from the typicality inclusions in the $ {\cal T}_{C_i}$'s.

\subsection{A semantics closure construction for weighted knowledge bases}  \label{sec:integer_weights}

Given a weighted knowledge base $K=\langle  {\cal T}_{strict}, {\cal T}_{C_1}, \ldots,$ $ {\cal T}_{C_k}, {\cal A}  \rangle$, where ${\cal T}_{C_i}=\{(d^i_h,w^i_h)\}$ for $i=1,\ldots,k$, 
and an 
$\el$ interpretation $I=\langle \Delta, \cdot^I \rangle$ satisfying all the strict inclusions  in ${\cal T}_{strict}$ and assertions in ${\cal A}$, we define a preference relation $<_{C_j}$ on $\Delta$ for each distinguished concepts $C_i \in {\cal C}$ through a {\em semantic closure construction}, a construction similar in spirit to the one considered by Lehmann for the lexicographic closure \cite{Lehmann95}, but based on a different seriousness ordering.
In order to define $<_{C_i}$ we consider the sum of the weights of the defeasible inclusions for $C_i$ satisfied by each domain element $x \in \Delta$; higher preference wrt $<_{C_i}$ is given to the domain elements whose associated sum (wrt $C_i$) is higher.

 %
 
First, let us define when a domain element $x \in \Delta$  satisfies/violates 
a typicality inclusion for $C_i$ wrt an $\el$ interpretation $I$. As $\el$ has the {\em finite model property} \cite{rifel}, we will restrict to $\el$  interpretations with a {\em finite} domain.
We say that {\em$ x \in \Delta$  satisfies $\tip(C_i) \sqsubseteq D$ in $I$}, if  $x   \not \in C_i^I$ or $x \in D^I$ (otherwise $x$ {\em violates} $\tip(C_i) \sqsubseteq D$ in $I$).
Note that, in an interpretation $I$, any domain element which is not an instance of $C_i$ trivially satisfies all defeasible inclusions $\tip(C_i) \sqsubseteq D$.
Such domain elements will be given the lowest preference with respect to $<_{C_i}$. 

Given an $\el$ interpretation $I=\langle \Delta, \cdot^I \rangle$ and a domain element $x \in \Delta$, we define  the {\em weight  of $x$ wrt $C_i$ in $I$} $W_i(x)$ 
considering the inclusions $(\tip(C_i) \sqsubseteq D_{i,h} \; , w^i_h)\in {\cal T}_{C_i}$:
\begin{align}\label{weight}
	W_i(x)  & = \left\{\begin{array}{ll}
						\sum_{{h: x \in D_{i,h}^I}} w_h^i & \mbox{ \ \ \ \  if }  x \in C_i^I \\
						- \infty &  \mbox{ \ \ \ \  otherwise }  
					\end{array}\right.
\end{align} 
where $-\infty$ is added at the bottom of all real values.

Informally, given an interpretation $I$, for  $x \in C_i^I$, the weight $W_i(x)$ of $x$ wrt $C_i$ is the sum of the weights of all the defeasible inclusions for $C_i$ satisfied by $x$ in $I$. The more plausible are the satisfied inclusions, the higher is the weight of $x$. 
For instance, in the example (Section \ref{sec:WeightedBKs}), assuming that domain elements $\mathit{tom,bob \in Employee^I}$, and that the typicality inclusion $(d_3)$ is satisfied by $\mathit{tom}$, while $(d_1),(d_2)$ are satisfied by $\mathit{bob}$, for $C_i=\mathit{Employee}$, we would get $W_i(tom)=-70$ and $W_i(bob)=100-70=30 $. 

Based on this notion of weight of a domain element wrt a concept, one can construct a preference relation $<_{C_i}$ from a given $\el$ interpretation $I$. A domain element $x$ is preferred to element $y$ wrt $C_i$ 
if the weight of the defaults in  ${\cal T}_{C_i}$ satisfied by $x$ is higher than weight of defaults in $ {\cal T}_{C_i}$ satisfied by $y$. 

\begin{definition}[Preference relation $<_{C_i}$ constructed from ${\cal T}_{C_i}$] \label{total_preorder}
Given a ranked knowledge base $K$ where, for all $j$, ${\cal T}_{C_j}=\{(d^i_h,r^i_h)\}$,
and an $\el$ interpretation $I=\langle \Delta, \cdot^I \rangle$, 
a  {\em preference relation $\leq_{C_i}$}  can be defined as follows:
For $x,y \in \Delta$,
\begin{align}\label{pref_rel}
x & \leq_{C_i}  y  \mbox{  \ \ iff \ \ }  W_i(x) \geq W_i(y)
\end{align}
\end{definition}
$\leq_{C_j}$  is a total preorder relation on $\Delta$. A strict preference relation (a strict modular partial order)  $<_{C_j}$ and  an equivalence relation $\sim_{C_j}$ can be defined on $\Delta$ by letting: $x <_{C_j} y$ iff ($x \leq_{C_j} y$ and  not $y \leq_{C_j} x$), and
$x \sim_{C_j} y$ iff ($x \leq_{C_j} y$ and $y \leq_{C_j} x)$.
Note that the domain elements which are instances of $C_i$  are all preferred (wrt  $<_{C_i}$) to the domain elements which are not instances of $C_i$. 
Furthermore, for all domain elements $x,y \not \in C_j^I$, $x \sim_{C_j} y$ holds.
The higher is the weight of an element wrt $C_i$ the more preferred is the element.
In the example, $W_i(bob)=30 > W_i(tom)=-70$  (for $C_i= \mathit{Employee}$) and, hence, $\mathit{bob <_{Employee} tom}$, i.e., Bob is more typical than Tom as an employee.

Following the same approach as for ranked $\el$ knowledge bases \cite{iclp2020}, we define a notion of 
{\em cw$^m$-model} for a weighted knowledge base $K$,
where each preference relation $<_{C_j}$ in the model is constructed from the  weighted TBox ${\cal T}_{C_j}$ according to Definition \ref{total_preorder} above, and the global preference is defined by combining the $<_i$'s. 
\begin{definition}[cw$^m$-model of $K$]\label{cwm-model} 
Let 
$K=\langle  {\cal T}_{strict},$ $ {\cal T}_{C_1}, \ldots,$ $ {\cal T}_{C_k}, {\cal A}  \rangle$ be a weighted $\el$ knowledge base over  ${\cal C}$,  
and 
$I=\langle \Delta, \cdot^I \rangle$ an $\el$ interpretation for $K$.
A {\em concept-wise multipreference model} {\em (cw$^m$-model)}  of $K$ is  a cw$^m$-interpretation ${\emme}=\langle \Delta,<_{C_1}, \ldots, <_{C_k}, <, \cdot^I \rangle$ 
such that: 
$\emme$  satisfies  all strict inclusions in $ {\cal T}_{strict}$ 
and assertions in ${\cal A}$, and
for all $j= 1, \ldots, k$,  $<_{C_j}$ is 
defined from  ${\cal T}_{C_j}$ and $I$, according to Definition \ref{total_preorder}. 
\end{definition}
As  preference relations $<_{C_j}$,  defined according to Definition \ref{total_preorder}, 
are irreflexive, transitive, modular, and  well-founded relations over $\Delta$ (for well-foundedness, remember that we are considering finite models),
the notion of cw$^m$-model $\emme$ introduced above is well-defined. 
By definition of cw$^m$-model, $\emme$ must satisfy all strict inclusions  
and assertions in $K$, but it is not required to satisfy all typicality inclusions $\tip(C_j) \sqsubseteq  D$ in $K$,
unlike other preferential typicality logics \cite{lpar2007}. This happens in a similar way in the multipreferential semantics for $\elpb$ ranked knowledge bases, and we refer to \cite{iclp2020} for an example showing that the cw$^m$-semantics is more liberal (in this respect) than  standard KLM preferential semantics. 

Observe that the notion of weight $W_i(x)$ of $x$ wrt $C_i$,  defined above as the sum of the weights of the satisfied defaults, is just a possible choice for the definition of the preference relations $<_i$ with respect to a concept $C_i$. 
 A different notion of preference $<_{C_i}$ has been defined from a ranked TBox ${\cal T}_{C_j}$ \cite{iclp2020}, 
by exploiting the (positive) integer ranks of the defeasible inclusions in ${\cal T}_{C_j}$ and the   (lexicographic) $\#$ strategy in 
the framework of basic preference descriptions  \cite{Brewka04}. The sum of the ranks has been first used in Kern-Isberner's c-interpretations  \cite{Kern-Isberner01,Kern-Isberner2014},
also considering 
the sum of the weights $\kappa_i^- \in \mathbb{N}$, representing penalty points for {\em falsified} conditionals. Here, we only sum the (positive or negative) weights of the satisfied defaults,  and we do it in a concept-wise manner. 

A notion of {\em concept-wise entailment}  (or cw$^m$-entailment) can be defined in a natural way to establish when a defeasible concept inclusion follows from a weighted knowledge base $K$. 
We can restrict our consideration 
 to (finite) {\em canonical } models, i.e., models which are large enough to contain 
all the relevant domain elements\footnote{This is a standard assumption in the semantic characterizations of rational closure for DLs, and in other semantic constructions. 
See  \cite{iclp2020} for the definition of canonical models for $\el$.}.


\begin{definition}[cw$^m$-entailment] \label{cwm-entailment}
An inclusion $\tip(C) \sqsubseteq D$ is cw$^m$-entailed  
from a weighted knowledge base $K$ 
if $\tip(C) \sqsubseteq D$ is satisfied in all canonical cw$^m$-models  
$\emme$ of $K$.
\end{definition}
As for ranked $\el$ knowledge bases \cite{iclp2020}, it can be proved that this notion of cw$^m$-entailment for weigthed KBs satisfies the KLM postulates of a preferential consequence relation \cite{iclp2020}. This is an easy consequence of the fact that the global preference relation $<$, which is used to evaluate  typicality,  is a strict partial order. 
As $<$ is not necessarily modular, cw$^m$-entailment does not necessarily satisfy rational monotonicity  \cite{whatdoes}.

The problem of deciding cw$^m$-entailment is $\Pi^p_2$-complete for ranked $\elpb$ knowledge bases \cite{iclp2020};
cw$^m$-entailment can be proven as well to be in $\Pi^p_2$ for weighted knowledge bases, 
based on a similar reformulation 
of cw$^m$-entailment as a problem of computing preferred answer sets. The proof of the result is similar to the proof of Proposition 7 in the online Appendix  of \cite{iclp2020},
apart from minor differences due to the different notion of preference $<_{C_i}$ used here with respect to the one for ranked knowledge bases. 

\section{Weighted Tboxes and multipreference fuzzy interpretations} \label{sec:fuzzyWeightedKB}

In this section, we move to consider fuzzy interpretations, 
and investigate the possibility of extending the previous multipreference semantic construction 
to  the fuzzy case.
%

  \begin{definition}[Fuzzy multipreference interpretation]\label{defi:fuzzy_multipreference}  
A {\em fuzzy multipreference interpretation} (or {\em fm-interpretation}) is a tuple
$\emme= \langle \Delta, <_{C_1}, \ldots, <_{C_k}, \cdot^I \rangle$, where:

(a) $(\Delta,\cdot^I)$ is a fuzzy interpretation;
 
(b) the $<_{C_i}$ are irreflexive, transitive, well-founded and modular relations over $\Delta$;


\end{definition}

%
%

%


Let $K$ be a weighted knowledge base $\langle  {\cal T}_{strict}, {\cal T}_{C_1}, \ldots, {\cal T}_{C_k}, {\cal A}  \rangle$, where  each axiom in ${\cal T}_{strict}$ has the form $\la \alpha \geq 1\ra$, and
${\cal T}_{C_i}=\{(d^i_h,w^i_h)\}$ is a set of typicality inclusions $d^i_h= \tip(C_i) \sqsubseteq D_{i,h}$ with weight $w^i_h$.


Given a fuzzy interpretation $I=\langle \Delta, \cdot^I \rangle$, satisfying all the strict inclusions  in ${\cal T}_{strict}$ and all assertions in ${\cal A}$,
we aim at constructing a concept-wise multipreference interpretation from $I$, 
by defining a preference relation $<_{C_j}$ on $\Delta$ for each $C_i \in {\cal C}$, based on a closure construction similar to the one developed in Section  \ref{sec:integer_weights}.
The definition of $W_i(x)$ in (\ref{weight}) can be reformulated as follows: 
 \begin{align}\label{weight_fuzzy}
	W_i(x)  & = \left\{\begin{array}{ll}
						 \sum_{h} w_h^i  \; D_{i,h}^I(x) & \mbox{ \ \ \ \  if } C_i^I(x)>0 \\
						- \infty &  \mbox{ \ \ \ \  otherwise }  
					\end{array}\right.
\end{align} 
by regarding the interpretation $D_{i,h}^I$ of concept $D_{i,h}$ as a two valued function from $\Delta$ to $\{0,1\}$ (rather than a subset of $\Delta$). And similarly for $C_i^I(x)$.
Definition (\ref{weight_fuzzy}) can be taken as the definition of the weight function $W_i(x)$ when $I$ is a fuzzy interpretation. 
Simply, in the fuzzy case,
for each default $d^i_h= \tip(C_i) \sqsubseteq D_{i,h}$, $D_{i,h}^I(x)$ is a value in $[0,1]$. 
In the sum, the value $D_{i,h}^I(x)$ of the membership of $x$ in $D_{i,h}$ is weighted by $ w_h^i $.
For inclusions $\tip(C_i) \sqsubseteq D_{i,h}$ with a positive weight, the higher is the degree of truth of the membership of $x$ in $D_{i,h}$, the higher is the weight $W_i(x)$. 
For inclusions with a negative weight, the lower is the degree of truth of the membership of $x$ in $D_{i,h}$, the higher is the weight $ W_i(x)$.

From this notion of weight of a domain element $x$ wrt a concept $C_i \in {\cal C}$, 
the {\em preference relation $\leq_{C_i}$  associated with ${\cal T}_{C_j}$ in a fuzzy interpretation $I$} can be defined as in Section \ref{sec:integer_weights}: 
\begin{align}\label{pref_rel_fuzzy}
x  \leq_{C_i}  y & \mbox{  \ \ iff \ \ } W_i(x) \geq W_i(y)  
\end{align}
%
\noindent 
A notion of fuzzy multipreference model of a weighted KB can then be defined. 

\begin{definition}[fuzzy multipreference model of $K$]\label{fuzzy_cwm-model} 
Let 
$K=\langle  {\cal T}_{strict},$ $ {\cal T}_{C_1}, \ldots,$ $ {\cal T}_{C_k}, {\cal A}  \rangle$ be a weighted $\el$ knowledge base over  ${\cal C}$.   
%
A {\em  fuzzy multipreference model} (or {\em fm-model})  of $K$ is  an fm-interpretation ${\emme}=\langle \Delta,<_{C_1}, \ldots, <_{C_k}, \cdot^I \rangle$ 
such that: 
the fuzzy interpretation $I=(\Delta,  \cdot^I)$  satisfies  all strict inclusions in $ {\cal T}_{strict}$ and assertions in ${\cal A}$
and, for all $j= 1, \ldots, k$,  $<_{C_j}$ is 
defined from  ${\cal T}_{C_j}$ and $I$, according to condition (\ref{pref_rel_fuzzy}). 
\end{definition}
Note that, as  we restrict to witnessed fuzzy interpretations $I$, for $S \neq \emptyset$, $\mathit{inf_{x \in S} C_i^I}$ is attained at some point in $\Delta$. 
Hence, $min_{<_{C_i}}(S) \neq \emptyset$, i.e., $<_{C_i}$ is well-founded.

The preference relation $<_{C_i}$ establishes how typical a domain element $x$ is 
wrt $C_i$. We can then require that the degree of membership in $C_i$ (given by the fuzzy interpretation $I$) and the relative typicality wrt $C_i$ (given by the preference relations $<_{C_i}$) are related, and agree with each other.

\begin{definition}[Coherent fm-models] \label{def_coherence} 
The preference relation $<_{C_i}$ {\em agrees with the fuzzy interpretation $I=\langle \Delta, \cdot^I \rangle$} if, for all  $x,y \in \Delta$: $x  <_{C_i}  y$  iff $C_i^I(x) > C_i^I(y)$. 
%

\noindent
An fm-model ${\emme}=\langle \Delta,<_{C_1}, \ldots, <_{C_k}, \cdot^I \rangle$ of $K$  is a {\em coherent} fm-model (or {\em cf$^m$-model}) of $K$ if, for all  
$C_i \in{\cal C}$,  preference relation $<_{C_i}$ agrees with the fuzzy interpretation $I$.
\end{definition}
In a cf$^m$-model, the preference relation $<_{C_i}$ over $\Delta$  constructed from ${\cal T}_{C_i}$  coincides with the preference relation induced by $C_i^I$.
As the interpretation function $\cdot^I$ extends to any concept $C$, for cf$^m$-models we do not need to introduce a global preference relation $<$, defined by combining the $<_{C_i}$.  To define the interpretation of typicality concepts $\tip(C)$ in a cf$^m$-model, we follow a different route and we let, for all concepts $C$,
 $$(\tip(C))^I= min_{<_C} (C^I),$$ 
 where $<_C$ is the preference relation over $\Delta$ induced by $C^I$, i.e., for all $x,y \in \Delta$: $x  <_{C}  y$  iff $C^I(x) > C^I(y)$.
Note that $\tip(C)$ is a two valued concept, i.e., $(\tip(C))^I(x) \in \{0,1\}$, and  satisfiability in a cf$^m$-model is now extended to fuzzy inclusion axioms  involving typicality concepts, such as $\langle \tip(C) \sqsubseteq D \geq \alpha \rangle$.

A notion of {\em cf$^m$-entailment} from a weighted knowledge base $K$ can be defined in the obvious way: a fuzzy axiom $E$   {is cf$^m$-entailed by a fuzzy knowledge base $K$} if, for all cf$^m$-models $\emme$ of $K$,
$\emme$ satisfies $E$.

\section{Preferential and fuzzy interpretations of multilayer perceptrons} \label{sec:multiulayer_perceptron}

In this section, we first shortly introduce multilayer perceptrons. Then we develop a preferential interpretation of a neural network after training, and a fuzzy preferential interpretation.


Let us first recall from \cite{Haykin99} the model of a {\em neuron} as an information-processing unit in an (artificial) neural network. The basic elements are the following:
\begin{itemize}
\item 
a set of {\em synapses} or {\em connecting links}, each one characterized by a {\em weight}. 
We let $x_j$ be the signal at the input of synapse $j$ connected to neuron $k$, and $w_{kj}$ the related synaptic weight;
\item
the adder for summing the input signals to the neuron, weighted by the respective synapses weights: $\sum^n_{j=1} w_{kj} x_j$;
\item
an {\em activation function} for limiting the amplitude of the output of the neuron (typically, to the interval $[0,1]$ or $[-1,+1]$).
\end{itemize}
The sigmoid, threshold and hyperbolic-tangent functions are examples of activation functions.
A neuron $k$ can be described by the following pair of equations: $u_k= \sum^n_{j=1} w_{kj} x_j $, and $y_k=\varphi(u_k + b_k)$,
where  $x_1, \ldots, x_n$ are the input signals and $w_{k1}, \ldots,$ $ w_{kn} $ are the weights of neuron $k$; 
$b_k$ is the bias, $\varphi$ the activation function, and $y_k$ is the output signal of neuron $k$.
By adding a new synapse with input $x_0=+1$ and synaptic weight $w_{k0}=b_k$, one can write: 
\begin{equation} \label{eq:neuron2}
u_k= \sum^n_{j=0} w_{kj} x_j \mbox{ \ \ \ \ \ \ \ \ \ \ \ \ \ \ } y_k=\varphi(u_k),
\end{equation}
where $u_k$ is called the {\em induced local field} of the neuron.
The neuron can be represented as a directed graph, where the input signals $x_1, \ldots, x_n$ and the output signal $y_k$ of neuron $k$ are nodes of the graph.
An edge from $x_j$ to $y_k$, labelled $w_{kj}$, means that  $x_j$ is an input signal of neuron $k$ with synaptic weight $w_{kj}$.  

A neural network can then be seen as ``a directed graph consisting of nodes with interconnecting synaptic and activation links"  \cite{Haykin99}:
nodes in the graph are the neurons (the processing units) 
and the weight $w_{ij}$ on the edge from node $j$ to node $i$ represents ``the strength of the connection [..] by which unit $j$ transmits information to unit $i$" \cite{Plunkett98}.
Source nodes (i.e., nodes without incoming edges) produce the input signals to the graph. 
Neural network models are classified by their synaptic connection topology. In a {\em feedforward} network the {architectural graph} is acyclic, while in a {\em recurrent} network it contains cycles. In a feedforward network neurons are organized in layers. In a {\em single-layer} network there is an input-layer of source nodes and an output-layer of computation nodes. In a {\em multilayer feedforward} network there is one or more hidden layer, whose computation nodes are called {\em hidden neurons} (or hidden units).
The source nodes in the input-layer supply the activation pattern ({\em input vector}) providing the input signals for the first layer computation units.
In turn, the output signals of first layer computation units provide the input signals for the second layer computation units, and so on, up to the final output layer of the network, which provides the overall response of the network to the activation pattern.
In a recurrent network at least one feedback 
exists, so that ``the output of a node in the system influences in part the input applied to that particular element"  \cite{Haykin99}.   
In the following, we do not put restrictions on the topology the network.

``A major task for a neural network is to learn a model of the world"  \cite{Haykin99}. In supervised learning, a set of input/output pairs, input signals and corresponding desired response, referred as training data, or training sample, is used to train the network to learn. In particular, the network learns by changing the synaptic weights, through the exposition to the training samples. After the training phase, in the generalization phase, the network is tested with data not seen before. ``Thus the neural network not only provides the implicit model of the environment in which it is embedded, but also performs the information-processing function of interest"  \cite{Haykin99}.
In the next section, we try to make this model explicit as a multipreference model.


\subsection{A multipreference interpretation of multilayer perceptrons}  \label{sec:sem_for_NN}

Assume that the network ${\cal N}$ has been trained and the synaptic weights $w_{kj}$ have been learned.
We associate a concept name $C_i \in N_C$ to any unit 
$i$ in ${\cal N}$ (including input units and hidden units) 
and construct a multi-preference interpretation over a (finite) {\em domain $\Delta$} of input stimuli,
the input vectors considered so far, for training and generalization.
In case the network is not feedforward, we assume that, for each input vector $v$ in $\Delta$, the network reaches a stationary state  \cite{Haykin99}, in which $y_k(v)$ is the activity level of unit $k$. 

Let ${\cal C}= \{ C_1, \ldots, C_n\}$ be a subset of $N_C$, the set of concepts $C_i$ for a distinguished subset of units $i$, 
the units we are focusing on (for instance, ${\cal C}$ might be associated to the set of output units, or to all units). 
We can associate to ${\cal N}$ and $\Delta$ a (two-valued) concept-wise multipreference interpretation over the boolean fragment of $\alc$ (with no roles and no individual names), based on Definition \ref{defi:cw-multipref}, as follows:

\begin{definition}
The  {\em cw$^m$interpretation $\emme_{\enne}^\Delta=\langle \Delta,<_{C_1}, \ldots, <_{C_n}, <, \cdot^I \rangle$ over $\Delta$
for network ${\cal N}$}  wrt ${\cal C}$ is a cw$^m$-interpretation where: \\
$-$ the interpretation function $\cdot^I$ is defined for named concepts $C_k \in N_C$ as: $x \in C_k^I$ if  $y_k(x) \neq 0$, and $x \not \in C_k^I$ if  $y_k(x) = 0$.\\
$-$  for $C_k \in {\cal C}$,  relation $<_{C_k}$ is defined for $x,x' \in \Delta$ as:
 $x <_{C_k} x'$ iff $y_k(x) > y_k(x')$\footnote{$y_k(x)$ is the output signal of unit $k$ for input vectors $x$. Differently from condition (\ref{eq:neuron2}), here (and below) the dependency of the output $y_k$ of neuron $k$ on the input vector $x$ is made explicit.}.
\end{definition}
The relation $<_{C_k}$ is a strict partial order, and $\leq_{C_k}$ and $\sim_{C_k}$ are defined as usual. In particular,  $x \sim_{C_k} x'$ for $x,x' \not \in C_k^I$.
Clearly, the boundary between the domain elements which are in $C_k^I$ and those  which are not could be defined differently, 
e.g., by letting $x \in C_k^I$ if  $y_k(x) > 0.5$, and $x \not \in C_k^I$ if  $y_k(x) \leq 0.5$. This would require only a minor change in the definition of the $<_{C_k}$. 

This model provides a multipreference interpretation of the network $\enne$, 
based on the input stimuli considered in $\Delta$.
For instance, when the neural network is used for categorization and a single output neuron is associated to each category, each concept $C_h$ associated to an output unit $h$  corresponds to a learned category. If $C_h \in {\cal C}$, the preference relation $<_{C_h}$ determines the relative typicality of input stimuli wrt category $C_i$. This allows to verify typicality properties concerning categories,  such as $\tip(C_h) \sqsubseteq D$ (where $D$ is a boolean concept built from the named concepts in $N_C$), by {\em model checking} on the model $\emme_{\enne}^\Delta$. According to the semantics of typicality concepts, this would require to identify typical $C_h$-elements 
and checking whether they are instances of concept $D$. 
%
%
General typicality inclusion of the form $\tip(C) \sqsubseteq D$, with $C$ and $D$ boolean concepts, can as well be verified on the model $\emme_{\enne}^\Delta$. 
However, the identification of $<$-minimal $C$-elements requires computing, for all pairs of elements $x,y \in \Delta$, the relation $<$ and the relations $<_{C_i}$ for $C_i \in {\cal C}$.  This may be challenging as $\Delta$ can be large. 
 
Evaluating properties involving hidden units might be of interest, although their meaning is usually unknown.  
In the well known Hinton's family example \cite{Hinton1986}, one may want to verify whether, normally, given an old Person 1 and relationship Husband, Person 2 would also be old, i.e., $\tip(Old_1 \sqcap Husband) \sqsubseteq Old_2$ is satisfied. 
Here, concept $Old_1$ (resp., $Old_2$) is  associated to a (known, in this case) hidden unit for Person 1 (and Person 2), while Husband is associated to an  input unit.

\subsection{A fuzzy interpretation of multilayer perceptrons}  \label{sec:fuzzy_sem_for_NN}

The definition of a fuzzy model of a neural network ${\cal N}$, under the same assumptions as in Section  \ref{sec:sem_for_NN},  is straightforward.
Let $N_C$ 
be the set containing a concept name $C_i$ for each unit $i$ in $\enne$, including hidden units.
Let us restrict to the boolean fragment of $\alc$ with no individual names.
We define a {\em fuzzy interpretation} $I_{\enne}=\langle \Delta, \cdot^I \rangle$ for $\enne$ as follows:
\begin{itemize}
\item
$\Delta$ is a (finite) set of input stimuli;

\item
the interpretation function $\cdot^I$ is defined for named concepts $C_k \in N_C$ as: $C_k^I(x)= y_k(x)$, $ \forall x \in \Delta$;
where $y_k(x)$ is the output signal of neuron $k$, for input vector $x$.
\end{itemize}
%
The verification that a fuzzy axiom $\la C \sqsubseteq D \geq  \alpha \ra$ is satisfied in the model $I_{\enne}$, can be done based on satisfiability in fuzzy DLs, according to the choice of the t-norm and implication function. It requires $C_k^I(x)$ to be recorded for all $k=1,\ldots, n$ and $x \in \Delta$.
Of course, one could restrict $N_C$  to the concepts associated to input and output units in $\enne$, so to capture the input/output behavior of the network.

In the next section, starting from this fuzzy interpretation of a neural network ${\enne}$, we define a fuzzy multipreference interpretation ${\emme^{f,\Delta}_{\enne}}$, and prove that ${\emme^{f,\Delta}_{\enne}}$ is a coherent fm-model of the conditional knowledge base $K_{\enne}$  associated to  ${\enne}$, under some condition.

\subsection{Multilayer perceptrons as conditional knowledge bases}  \label{sec:NN&Conditionals} 


Let $N_C$ be as in Section  \ref{sec:fuzzy_sem_for_NN}, and
let ${\cal C}= \{ C_1, \ldots, C_n\}$ be a subset of $N_C$.
Given the {\em fuzzy interpretation} $I_{\enne}=\langle \Delta, \cdot^I \rangle$  as defined in Section  \ref{sec:fuzzy_sem_for_NN},
%
%
a fuzzy multipreference interpretation ${\emme^{f,\Delta}_{\enne}}=\langle \Delta,<_{C_1}, \ldots, <_{C_n}, \cdot^I \rangle$ over $\cal C$
can be defined 
 by letting  $<_{C_k}$ to be the preference relation induced by the interpretation $I_{\enne}$, as follows: for $x,x' \in \Delta$,
\begin{equation} \label{def:pref_M_N}
 x <_{C_k} x' \mbox{ iff } y_k(x) > y_k(x').
\end{equation}
Interpretation ${\emme^{f,\Delta}_{\enne}}$ makes the preference relations induced by $I_{\enne}$ explicit. 
We aim at proving that ${\emme^{f,\Delta}_{\enne}}$ is indeed a coherent fm-model of the neural network ${\enne}$. 
A weighted conditional knowledge base $K^{\enne}$ is associated to the neural network ${\enne}$ as follows.

For each unit $k$, we consider all the units $j_1, \ldots, j_m$ whose output signals are the input signals 
of unit $k$, with synaptic weights $w_{k,{j_1}}, \ldots, w_{k,{j_m}}$.  Let $C_k$ be the concept name associated to unit $k$ and $C_{j_1}, \ldots, C_{j_m}$ the concept names associated to units $j_1, \ldots, j_m$, respectively.
We define for each unit $k$ the following set ${\cal T}_{C_k}$ of typicality inclusions, with their associated weights:
\begin{quote}
$\tip(C_k) \sqsubseteq C_{j_1}$ with  $w_{k,{j_1}}$, 
$\ldots$ ,
$\tip(C_k) \sqsubseteq C_{j_m}$ with  $w_{k,{j_m}}$
\end{quote}
Given ${\cal C}$, the knowledge base  extracted from network ${\enne}$ is defined as the tuple: $K^{\enne} = \langle  {\cal T}_{strict},{\cal T}_{C_1}, \ldots,$ $ {\cal T}_{C_n}, {\cal A}  \rangle$, where $ {\cal T}_{strict}= {\cal A}=\emptyset$ and $K^{\enne} $ contains the set  ${\cal T}_{C_k}$ of weighted typicality inclusions associated to neuron $k$ (defined as above),  for each $C_k \in {\cal C}$.
$K^{\enne}$ is a weighted knowledge base over the set  of distinguished concepts ${\cal C}= \{ C_1, \ldots, C_n\}$.
For multilayer feedforward networks, $K^{\enne}$ correspond to an acyclic conditional knowledge base, 
and defines a (defeasible) subsumption hierarchy among concepts. 
It can be proved that:

\begin{proposition} \label{Prop}
${\emme^{f,\Delta}_{\enne}}$ is a cf$^m$-model of the  knowledge base $K^{\enne}$, provided the activation functions $\varphi$ of all neurons are monotonically increasing and have value in $(0,1]$.
\end{proposition}
\begin{proof}
Let $\emme^{f,\Delta}_{\enne}=\langle \Delta,<_{C_1}, \ldots, <_{C_n}, <, \cdot^I \rangle$. 
Let us consider any neuron $k$, such that $C_k \in {\cal C}$.  

From the hypothesis, for any input stimulus $x \in \Delta$,  $y_k(x)  >0$ and, hence, $C_k(x)  >0$.
The weight  $W_k(x)$ of $x$ wrt $C_k$ is defined, according to equation (\ref{weight_fuzzy}),
as 
$$W_k(x)=  \sum_{h=1}^{m} w_{k,{j_h}} \; C_{j_h}^I(x)$$

\noindent
where $\tip(C_k) \sqsubseteq C_{j_1}$ 
$\ldots$
$\tip(C_k) \sqsubseteq C_{j_m}$ are the typicality inclusions in  ${\cal T}_{C_k}$ with weights $w_{k,{j_1}}, \ldots, w_{k,{j_m}}$.
Observe that,  for all $h=1,\ldots, m$, $C_{j_h}^I(x)= y_{j_h}$, for input $x$,  by construction of interpretation $\emme^{f,\Delta}_{\enne}$ (and of $\cdot^I$). Therefore:

$$W_k(x)=  \sum_{h=1}^{m} w_{k{j_h}}  y_{j_h}$$

\noindent
As $y_{j_1}, \ldots, y_{j_m}$ are the input signals 
of unit $k$, 
it holds that
$W_k(x)$ is the induced local field $u_k$ of unit $k$ (as in equation  (\ref{eq:neuron2})), for the input stimulus $x$.
As $u_k=W_k(x)$ than, for the given input stimulus $x$, the output  of neuron $k$ must be $y_k(x)=\varphi(u_k)=\varphi(W_k(x))$, where $\varphi$ is the activation function of unit $k$. As by construction of $\emme^{f,\Delta}_{\enne}$,  $C_{k}^I(x)= y_{k}(x)$ for all units $k$,  it holds that $C_{k}^I(x) =\varphi(W_k(x))$. 

To prove that $\emme^{f,\Delta}_{\enne}$ is a cf$^m$-model of $K^\enne$, we have to prove that $\emme^{f,\Delta}_{\enne}$ is a fuzzy multipreference model of $K_\enne$ and it is coherent.

We first prove that $\emme^{f,\Delta}_{\enne}$ is an fm-model of $K_\enne$.
As ${\cal T}_{Strict}$ and ${\cal A}$ are empty in $K^\enne$.
We only have to prove that, for all $C_k \in {\cal C}$, $<_{C_k}$ satisfies Condition (\ref{pref_rel_fuzzy}). We prove that,
\begin{center}
 $x <_{C_k} x'$ iff $W_k(x) > W_k(x')$,
\end{center}
from which (\ref{pref_rel_fuzzy}) follows.
For all $x, x' \in \Delta$, by construction of $\emme^{f,\Delta}_{\enne}$,
\begin{center}
 $x <_{C_k} x'$ iff $y_k(x) > y_k(x')$
\end{center}
Assume that $x <_{C_k} x'$. 
%
As $y_k(x)=\varphi(W_k(x))$ and $y_k(x')=\varphi(W_k(x'))$,  $\varphi(W_k(x)) > \varphi(W_k(x'))$ holds.
Then, it must be the case that $W_k(x) > W_k(x')$ (otherwise, by the assumption that $\varphi$ is monotone increasing, 
from $W_k(x) \leq W_k(x')$ it would follow that  $\varphi(W_k(x)) \leq \varphi(W_k(x'))$).

Conversely, assume that  $W_k(x) > W_k(x')$. As, from the hypothesis, $C_k(x') = y_k(x') >0$,  
both $W_k(x)$ and $W_k(x')$ are weighted sum of real numbers.
As $\varphi$ is monotonically increasing, $\varphi(W_k(x)) > \varphi(W_k(x'))$
and, hence, $y_k(x) > y_k(x')$, so that  $x <_{C_k} x'$ holds.  

It is easy to prove that $\emme^{f,\Delta}_{\enne}$ is coherent. 
By construction of $\emme^{f,\Delta}_{\enne}$, for each $C_k \in {\cal C}$, $<_{C_k}$ is defined by equation (\ref{def:pref_M_N}) as: 
 \begin{align*}
 x <_{C_k} x' \mbox{ iff } y_k(x) > y_k(x')
\end{align*}
As $C_{k}^I(x)= y_{k}(x)$ and $C_{k}^I(x')= y_{k}(x')$ (again by construction of  $\emme^{f,\Delta}_{\enne}$),
\begin{align*}
 x <_{C_k} x' \mbox{ iff } C^I_k(x) > C^I_k(x')
\end{align*}
i.e., $\emme^{f,\Delta}_{\enne}$ is a coherent fm-model of $K^\enne$.
\qed
\end{proof}
%

\noindent
Under the given conditions, that hold, for instance, for the sigmoid activation function,
for any choice of ${\cal C} \subseteq N_C$  and for any choice of the domain $\Delta$ of input stimuli (all leading to a stationary state of $\enne$), the fm-interpretation ${\emme^{f,\Delta}_{\enne}}$ 
is a coherent fuzzy multipreference model of the defeasible knowledge base $K^{\enne}$. 
 The knowledge base $K^{\enne}$ does not provide a logical characterization of the neural network $\enne$, as the requirement of coherence does not determine  the activation functions of  neurons.
For this reason, the knowledge base $K^{\enne}$ 
captures the behavior of  
all the networks ${\enne'}$, 
obtained from $\enne$ by replacing  the activation function of the units in $\enne$ with other monotonically increasing activation functions  with values in $(0,1]$, in all possible ways (but retaining the same synaptic weights as in $\enne$).
That is, an interpretation ${\emme^{f,\Delta}_{\enne'}}$,  constructed from a network $\enne'$ and any $\Delta$ as above, 
is as well a cf$^m$-model of $K^{\enne}$. 
This means that 
the logical formulas cf$^m$-entailed from $K^{\enne}$ hold in all the models ${\emme^{f,\Delta}_{\enne'}}$ built from $\enne'$. They are properties of $\enne'$, as well as of network $\enne$. cf$^m$-entailment from $K^{\enne}$ is sound for $\enne$ and for each $\enne'$ as above. 



\section{Weak coherence and monotonically non-decreasing activation functions}

In this section we aim at weakening the coherence requirement for a fuzzy multipreference interpretation in order to capture a wider class of monotone non-decreasing activation functions.

Let us define a notion of weak coherence of a fuzzy multipreference model $\emme$ of a knowledge base $K$ with respect to a fuzzy interpretation $I=\langle \Delta, \cdot^I \rangle$.

\begin{definition}[Weakly coherent fm-models] \label{def_weak_coherence} 
The preference relation $<_{C_i}$ {\em weakly agrees with the fuzzy interpretation $I=\langle \Delta, \cdot^I \rangle$} if, for all  $x,y \in \Delta$: 
\begin{align}\label{weak_agreement}
 C_i^I(x) > C_i^I(y) \; \Ri \; x & <_{C_i}  y  
 \end{align}

\noindent
An fm-model ${\emme}=\langle \Delta,<_{C_1}, \ldots, <_{C_k}, \cdot^I \rangle$ of $K$  is a {\em weakly coherent} fm-model (or {\em cf$^m$-model}) of $K$ if, for all  
$C_i \in{\cal C}$,  preference relation $<_{C_i}$ weakly agrees with the fuzzy interpretation $I$.
\end{definition}
Weak coherence only requires that the preference relation $<_{C_i}$ {\em respects} the preference ordering induced by the fuzzy interpretation $I$, without requiring that they coincide. $<_{C_i}$ can be finer but cannot be coarser than the preference ordering induced by the fuzzy interpretation $I$.

Let $\enne$ be network such that the activation functions $\phi$ of all neurons in $\enne$ are monotone non-decreasing with a value in $[0,1]$. 
Let $N_C$ be as in Section  \ref{sec:fuzzy_sem_for_NN}, and
let ${\cal C}= \{ C_1, \ldots, C_n\}$ be a subset of $N_C$.
Given the {\em fuzzy interpretation} $I_{\enne}=\langle \Delta, \cdot^I \rangle$  as defined in Section  \ref{sec:fuzzy_sem_for_NN},
where $C_k^I(x)=y_k(x)$ for all $x \in \Delta$,
%
%
a fuzzy multipreference interpretation ${\emme^{*,\Delta}_{\enne}}=\langle \Delta,<_{C_1}, \ldots, <_{C_n}, \cdot^I \rangle$ over $\cal C$
can be defined 
 by letting  $\leq_{C_k}$ be the preference relation defined as follows: for $x,x' \in \Delta$,
\begin{equation} \label{def:pref_M_N}
 x \leq_{C_k} x' \mbox{ iff } W_k(x) \geq W_k(x')
\end{equation}
and $<_{C_k}$ the associated strict partial order.
By construction, the model ${\emme^{*,\Delta}_{\enne}}$ satisfies Condition (\ref{pref_rel_fuzzy}), and is an fm-model of $K_\enne$.

It is easy to see that this model is weakly coherent, i.e., for all $x,y \in \Delta$,
 $ C_i^I(x) > C_i^I(y) \Ri x  <_{C_i}  y$.
 
From the equivalence $C_i^I(x) = y_i(x)= \varphi(W_i(x))$, from $ C_i^I(x) > C_i^I(y)$, it follows that   $\varphi(W_k(x)) > \varphi(W_k(y))$.

As $\varphi$ is a monotone non-decreasing activation function, then $W_k(x) > W_k(y)$ must hold.
By construction of ${\emme^{*,\Delta}_{\enne}}$, $x<_{C_i} y$. The next proposition follows.

\begin{proposition} \label{Prop1}
${\emme^{f,\Delta}_{\enne}}$ is a cf$^m$-model of the  knowledge base $K^{\enne}$, provided the activation functions $\varphi$ of all neurons are monotone non-decreasing.
\end{proposition}

%
%
%

\section{Towards a probabilistic account in probabilistic DLs}

In the previous section,  for an input vector $x \in \Delta$ and a unit $i$, we have interpreted  $C_i(x)$ as the degree of membership of $x$ in the concept $C_i$ in a fuzzy DL interpretation $I_\enne$. 
%
%
In this section, we aim at  discussing whether a probabilistic DL interpretation for of the neural network $\enne$ can be defined starting form the fuzzy DL interpretation $I_\enne$ introduced in Section \ref{sec:fuzzy_sem_for_NN}. 
Among the probabilistic extensions of description logics \cite{Lukasiewicz08,LutzS_KR10,LutzS_JAIR17,Penaloza2017,Wilhelm2019,BaaderEKW19}, we will consider those based on the statistical approach as well as those based on the subjective approach. Both approaches have been considered for DLs,  stemming from 
Halpern's Type 1 and Type 2 probabilistic FOL \cite{Halpern1990} for formalizing statistical and subjective probabilities.  

We follow the proposal by Zadeh, who showed that 
"the notions of an event and its probability can be extended in a natural fashion to fuzzy events" \cite{Zadeh1968}.
Given a probability space $(\mathbb{R}^n, {\cal A}, P)$, in which ${\cal A}$ is the $\sigma$-field of Borel sets in $\mathbb{R}^n$ and $P$ is a probability measure over $\mathbb{R}^n$,
Zadeh defines a {\em fuzzy event} in $\mathbb{R}^n$  as a fuzzy set in $\mathbb{R}^n$ whose membership function $\mu_A$ (with $\mu_A: \mathbb{R}^n \ri [0,1]$) is Borel measurable. He defines the {\em probability of a fuzzy event} $A$ by the Lebesgue-Stieltjes integral:

\begin{align*}
P(A)= & \int_{\mathbb{R}^n} \mu_A(x)  dP\\
= & \; E(\mu_A)
\end{align*}
where the probability of a fuzzy event is the expectation of its membership function.
Zadeh proves that the set of fuzzy events forms a $\sigma$-field with respect to the operations of complement, union and intersection in (Zadeh's) fuzzy logic.

Let us restrict to Zadeh's fuzzy logic. We adopt Zadeh's notion of probability of a fuzzy event to build a probabilistic $\alc$ interpretation starting from the fuzzy $\alc$ interpretation $I_\enne$.
As a basis for the definition of a probabilistic $\alc$ interpretation, we exploit the approach described by Lutz and Schr\"oder in the appendix of their work on the Prob-$\alc$ family of probabilistic DLs \cite{LutzS_KR10}. The appendix includes the definition of a probabilistic description logic Prob1-$\alc$ based on Halpern's Type 1 logics for statistical probabilities in FOL. The description logic Prob1-$\alc$ is intended to capture statistical probabilities in a DL using probability distributions on the domain $\Delta$. 
%
%
Lutz and Schr\"oder show that  
Prob1-$\alc$ is of very limited expressive power. 
In the following, we exploit their simple construction for defining Type 1 probabilistic $\alc$ interpretations,
but starting from fuzzy $\alc$ interpretations rather than from two-valued ones.  

As for the interpretation $I_{\enne}=\langle \Delta, \cdot^I \rangle$ constructed from the neural network $\enne$, in the following we will consider fuzzy interpretations over a finite domain $\Delta$ (a finite set of input stimuli), 
in which each concept $C$ is interpreted as a fuzzy set with  membership function $C^I: \Delta \ri [0,1]$. 

We introduce a logic ProbF-$\alc$, whose concepts are defined as in $\alc$. 
A TBox in ProbF-$\alc$, includes fuzzy concept inclusions as well as 
{\em probabilistic conditional constraints} $(C | D)[l,u]$, where $C$ and $D$ are concepts and $l,u$ are reals in $[0,1]$, meaning that  the probability of $C$ given $D$ lies between $l$ and $u$. Such conditional constraints are similar to those considered for the two-valued case in \cite{LukasiewiczStraccia09,Penaloza2017,Wilhelm2019}), while here concepts $C$ and $D$ are interpreted as being fuzzy.
As an example, the conditional constraint 
$(Quiet | Diligent )[0.4,0.8]$ has the intended meaning that the probability of an individual  being $Quiet$ given that he/she is $Diligent$ lies between $0.4$ and $0.8$, where $Quiet$ and $ Diligent$ have a fuzzy interpretation.

A {\em fuzzy-probabilistic $\alc$ interpretation} over a finite domain $\Delta$ is a structure 
$${\cal I}=\langle \Delta, \cdot^I , \mu \rangle$$
where
$\langle \Delta, \cdot^I \rangle$ is a fuzzy $\alc$ interpretation, as in Section \ref{sec:fuzzyDL}, and $\mu$ is a discrete probability distribution over $\Delta$.
In a  fuzzy-probabilistic interpretation ${\cal I}=\langle \Delta, \cdot^I , \mu \rangle$, the interpretation of concepts is defined as in the fuzzy interpretation $\langle \Delta, \cdot^I \rangle$  (see Section \ref{sec:fuzzyDL}) and does not depend on $\mu$. 
For each $\alc$ concept $C$, we let 
$\mu(C^I)$ denote 
$$\sum_{d \in \Delta} C^I(d) \; \mu(d)$$
where $C^I(d)$ is the degree of membership of $d$ in $C$. This is the point where we follow the proposal by Zadeh for defining the probability of a fuzzy event and diverge from Prob1-$\alc$ semantics 
in which $\mu(C^I)$ denotes  $\sum_{d \in C^I} \mu(d)$, where $C^I$ is crisp.

A Tbox is satisfied in $\cal I$ if all its axioms are satisfied in $\cal I$.
The satisfiability of fuzzy concept inclusions 
is defined as in the fuzzy interpretation  $\langle \Delta, \cdot^I \rangle$, i.e., $C \sqsubseteq D \; \theta \;  n$ is satisfied in ${\cal I}$ if  
$(C \sqsubseteq D)^I \; \theta \; n$.
The satisfiability of a conditional constraint is defined as follows:  $(C | D)[l,u]$ is satisfied in ${\cal I}$ if $\mu((C \sqcap D)^I) / \mu(D^I) \in [l,u]$. \normalcolor


Note that, by definition, 
\begin{equation}
\label{relative-distance}
\frac{ \mu((C \sqcap D)^I) }{ \mu(D^I)} =  \frac{ \sum_{d \in \Delta} (C \sqcap D)^I(d) \; \mu(d) } {\sum_{d \in \Delta}  D^I(d) \; \mu(d) }
\end{equation}
and, under the hypothesis that the probability distribution $\mu$ is uniform:
\begin{equation}
\frac{ \mu((C \sqcap D)^I) }{ \mu(D^I)} =  \frac{ \sum_{d \in \Delta}  (C \sqcap D)^I(d)} {\sum_{d \in \Delta}  D^I(d)} =\frac{ M ((C\sqcap D)^I)}{ M(D^I)}
\end{equation}
and the satisfiability of a conditional constraint $(C | D)[l,u]$ in an interpretation ${\cal I}$ 
can be evaluated by computing the ratio $M ((C\sqcap D)^I)/ M(D^I)$,
where 
$M(A)= \sum_{d \in \Delta}  A^I(d)$ is the size or cardinality of the fuzzy concept $A$ with membership function $A^I$. 

This is in agreement with Kosko's
account of "fuzziness in a probabilistic world" \cite{Kosko92}. He 
proved that ``the ratio form of the subsethood measure $S(A,B)$ has the same ratio form as the conditional probability $P(A|B)$", so that  ``subsethood reveals the connection between fuzziness and randomness" \cite{Kosko92}. In his Subsethood Theorem (see  \cite{Kosko92}, Chap. 7), he proved that $S(A,B)= M(A \cap B) / M(A)$, where $S(A,B)$ is the degree to which a fuzzy set $A$ belongs to a fuzzy set $B$ where, for a fuzzy set $A$ over a domain $X$ with membership function $m_A: X \ri [0,1]$,
$M(A)= \sum_{x \in X}  m_A(x)$ is the size of $A$. 


Let us consider the extension of $\alc$ with nominals, that is, with concepts of the form $\{a \}$ where $a$ is an individual name in $N_I$ representing a domain element.  In a standard DL interpretation a nominal $\{a \}$ is  interpreted  as a singleton,  i.e., $\{a \}^I =\{a^I\}$, 
and we will interpret it in the same way in a fuzzy interpretation $I=(\Delta, \cdot^I)$ (following Bobillo and Straccia in Fuzzy OWL 2 EL \cite{BobilloOWL2EL2018}). 
As $\{a \}^I$ is  crisp, its membership function  is the characteristic function:
 \begin{align*}\label{weight_fuzzy}
	\{a \}^I(d) & = \left\{\begin{array}{ll}
						 1 & \mbox{ \ \ \ \  if } a^I=d \\
						 0 &  \mbox{ \ \ \ \  otherwise }  
					\end{array}\right.
\end{align*} 
Let us assume that $N_I$ contains an individual name $x$ for each input stimulus  $x \in \Delta$, 
and let us consider, for some $x \in \Delta$, the conditional constraint $(C | \{x\} )[l,u]$ (we will simply write $(C | x)[l,u]$).

$(C | x)[l,u]$ is satisfied in ${\cal I}$ if $\mu((C \sqcap \{x\} )^I) / \mu(\{x\}^I) \in [l,u]$, where:
\begin{align*}
\frac{ \mu((C \sqcap \{x\})^I) }{ \mu(\{x\}^I)} & =  \frac{ \sum_{d \in \Delta} (C \sqcap \{x\})^I(d) \; \mu(d) } {\sum_{d \in \Delta}  \{x\}^I(d) \; \mu(d) } 
\end{align*}
As $(C \sqcap \{x\})^I(d)= 0$ for $d\neq x$, and that  $ \{x\}^I(d)=0$ for $d\neq x$:
\begin{align*}
\frac{ \mu((C \sqcap \{x\})^I) }{ \mu(\{x\}^I)} & =  \frac{  (C \sqcap \{x\})^I(x) \; \mu(x) } {  \{x\}^I(x) \; \mu(x) }
\end{align*}
Observing that $(C \sqcap \{x\})^I(x)= C^I(x)$ and  $ \{x\}^I(x)=1$:
\begin{align*}
\frac{ \mu((C \sqcap \{x\})^I) }{ \mu(\{x\}^I)} 
&=  \frac{  C^I(x) \; \mu(x) } { \mu(x) } 
=  C^I(x)
\end{align*}
Therefore, the conditional constraint $(C | x)[l,u]$  is satisfied in ${\cal I}$ when $C^I(x) \in [l,u]$. We are interpreting the membership value $C^I(x)$ in a fuzzy interpretation as the conditional probability of $C$ given the input stimulus $x$.  
This observation suggests an alternative way to develop a probabilistic $\alc$ interpretation from a fuzzy interpretation $I$ (like $I_\enne$), adopting a subjective view of probabilities in DLs.

The Prob-$\alc$ family \cite{LutzS_KR10,LutzS_JAIR17} is a family of probabilistic DLs which adopts a subjective view of probabilities as {\em degrees of belief}, and is concerned with probability distributions on a set of possible worlds (each one associated with a standard $\alc$ interpretation). Another approach is followed in the probabilistic description logic $\alc^{ME}$, introduced by Wilhelm  et al. \cite{Wilhelm2019,BaaderEKW19}, 
which combines subjective and statistical probabilities.
A probabilistic interpretation is defined a probability distribution $\mu: I_{K,\Delta} \ri [0,1]$ over the (finite) set $I_{K,\Delta}$ of all the standard $\alc$ interpretations with fixed finite domain $\Delta$.

Let us mention, as an example for subjective probability, the following one from \cite{BaaderEKW19}: ``a doctor may not know definitely that a patient has influenza, but only believe that this is the case with a certain probability".
The interpretation that the neural network gives to input stimuli can be regarded as well as subjective: the activation value of unit $k$ in the network $\enne$ for an input stimulus $x$ (e.g., an image) can be regarded as the membership degree of $x$ in concept $C_k$ but also as a degree of belief that $x$ is an instance of concept $C_k$.
%
%
In this view, for each input stimulus $x$ and each unit $k$ in $\enne$, the activation value $y_k$ of $k$ can than be regarded as the subjective probability that $x$ is an instance of $C_k$.
This can be expressed in the formalism by Wilhelm  et al. \cite{BaaderEKW19} as a  {\em probabilistic assertion} $P(C_k(x))[y_k]$, and in the formalism by Guti{\'{e}}rrez{-}Basulto et al. \cite{LutzS_JAIR17} as an Abox assertion $P_{=y_k}(C_k(x))$.
Given a probabilistic Abox ${\cal A}^\enne$ containing all the probabilistic assertions $P(C_k(x))[y_k]$, for all neurons $k$ and inputs $x$, the set of probabilistic interpretations 
satisfying ${\cal A}^\enne$ can than be regarded as the probabilistic interpretations of the neural network $\enne$.  
Whether such probabilistic DL interpretations are models of the neural network in a deeper sense, e.g., by considering $\enne$ as a set of probabilistic conditionals, has to be investigated and related to the major approaches from statistical relational AI (StarAI) \cite{DeRaedt2016}, 
which has strong relations to neural-symbolic computing \cite{SurveyDeRaedt2020}.

\section{Conclusions} 

In this paper, we have investigated the relationships between defeasible knowledge bases, under a fuzzy multipreference semantics, and multilayer neural networks. Given a network after training, we have seen that one can construct a (fuzzy) multipreference interpretation starting from a domain containing a set of input stimuli, and using the activity level of neurons for the stimuli.  
We have proven that such interpretations are models of the conditional knowledge base associated to the network, corresponding to a set of weighted defeasible inclusions in a simple DL. 

The correspondence between neural network models and fuzzy systems has been first investigated by Bart Kosko in his seminal work  \cite{Kosko92}.
In his view, ``at each instant the n-vector of neuronal outputs defines a fuzzy unit or a fit vector. Each fit value indicates the degree to which the neuron or element belongs to the n-dimentional fuzzy set." As a difference, our fuzzy interpretation of a multilayer perceptron regards each concept (representing a single neuron) as a fuzzy set. 
This is the usual way of viewing concepts in fuzzy DLs \cite{Straccia05,LukasiewiczStraccia08,BobilloStraccia16}, and we have interpreted concepts as fuzzy sets within a multipreference semantics based on a semantic closure construction, 
in the line of
Lehmann's semantics for lexicographic closure \cite{Lehmann95} and of Kern-Isberner's c-interpretations \cite{Kern-Isberner01,Kern-Isberner2014}. 

Much work has been devoted, in recent years, to the 
combination 
 of neural networks and symbolic reasoning, leading to the definition of new computational models \cite{GarcezLG2009,GarcezGori2019,GarcezGori2020,SurveyDeRaedt2020,SerafiniG16,Lukasiewicz2020}, 
 to extensions of logic languages 
with neural predicates \cite{DeepProbLog18,NeurASP2020} and vice-versa providing encoding of symbolic knowledge in neural architectures. 
Among the earliest 
systems combining logical reasoning and neural learning are the Knowledge-Based Artificial Neural Network (KBANN) \cite{KBANN94} and the Connectionist Inductive Learning and Logic Programming (CILP) \cite{CLIP99} 
systems. Penalty logic \cite{Pinkas95} , a non-monotonic reasoning formalism, was proposed as a mechanism to represent weighted formulas in energy-based connectionist (Hopfield) networks. 
%
Recent proposals for neural symbolic integration \cite{GarcezGori2020} include Logic Tensor networks \cite{SerafiniG16}, 
a generalization of the Neural Tensor Networks \cite{Socher2013}, and Graph Neural Networks \cite{Scarselli2009}. 
%
%
Here, rather than developing a new neural model to capture symbolic reasoning, 
we have provided a multipreference semantics for multilayer perceptrons as such, thus establishing a link between this neural network model and conditional reasoning.
%
This logical interpretation 
 may be of interest from the standpoint of explainable AI \cite{Adadi18,Guidotti2019,Arrieta2020}
and might be potentially exploited for an integrated use of neural network models and defeasible knowledge bases.  


Several issues may deserve further investigation as future work.
An open problem is whether the  the notion of cf$^m$-entailment is decidable  
(even for the small fragment of $\el$ without roles), under which choice of fuzzy logic combination functions, 
and whether decidable approximations can be defined. 
Another issue is whether the multipreference semantics can provide a semantic interpretation of other neural network models, besides
self-organising maps  \cite{kohonen2001}, whose multipreference semantics has been investigated in  \cite{CILC2020}.

\bibliography{biblioMultipreferenzeFuzzyProbNN2}
\end{document}